\documentclass[runningheads,final,hidelinks]{llncs}
\usepackage[T1]{fontenc}
\usepackage{graphicx}
\usepackage{xcolor}
\usepackage{amsmath, amssymb}
\usepackage{subcaption}
\usepackage{url}
\usepackage{booktabs}
\usepackage[ruled,vlined,linesnumbered]{algorithm2e}
\usepackage{algpseudocode}
\usepackage[
  backend=biber,
  style=numeric-comp,   
  sortcites=true,
  sorting=none,         
  natbib=true,          
  giveninits=true,
  maxcitenames=4,       
  minbibnames=3,        
  maxbibnames=4,        
  doi=false,            
  url=false,
  hyperref=true
]{biblatex}
\DeclareSourcemap{
  \maps[datatype=bibtex]{
    \map{
      \step[fieldset=file,      null]
      \step[fieldset=abstract,  null]
      \step[fieldset=keywords,  null]
      \step[fieldset=urldate,   null]
      \step[fieldset=month,     null]
      \step[fieldset=day,       null]
      \step[fieldset=language,  null]
      \step[fieldset=isbn,      null]
      \step[fieldset=issn,      null]
      \step[fieldsource=doi,
            match=\regexp{^\s*https?://(dx\.)?doi\.org/},
            replace={}]
      \step[fieldsource=booktitle,
            match=\regexp{^\s*Proceedings\s+of\s+(the\s+)?},
            replace={}]
      \step[fieldsource=eventtitle,
            match=\regexp{^\s*Proceedings\s+of\s+(the\s+)?},
            replace={}]
      \step[fieldsource=booktitle,
            match=\regexp{^\s*(19|20)\d{2}\s+},
            replace={}]
      \step[fieldsource=eventtitle,
            match=\regexp{^\s*(19|20)\d{2}\s+},
            replace={}]
      \step[fieldsource=booktitle,
            match=\regexp{^\s*\d+(st|nd|rd|th)\s+},
            replace={}]
      \step[fieldsource=eventtitle,
            match=\regexp{^\s*\d+(st|nd|rd|th)\s+},
            replace={}]
      \step[fieldsource=note,
            match=\regexp{.*arXiv:([0-9]{4}\.[0-9]{4,5})(v[0-9]+)?.*},
            replace={$1},
            fieldtarget=eprint]
      \step[fieldsource=eprint,
            match=\regexp{^[0-9]{4}\.[0-9]{4,5}(v[0-9]+)?$},
            fieldset=eprinttype,
            fieldvalue={arxiv}]
      \step[fieldsource=eprint,
            match=\regexp{^([0-9]{4}\.[0-9]{4,5}(v[0-9]+)?)$},
            replace={https://arxiv.org/abs/$1},
            fieldtarget=url]
    }
  }
}
\renewbibmacro{in:}{}

\newbibmacro*{title:doiurl}[1]{%
  \iffieldundef{doi}
    {\iffieldundef{url}
       {#1}
       {\href{\thefield{url}}{#1}}}
    {\href{https://doi.org/\thefield{doi}}{#1}}%
}

\DeclareFieldFormat{title}{%
  \usebibmacro{title:doiurl}{\mkbibemph{#1}}%
}
\DeclareFieldFormat[article,incollection,techreport,inproceedings,book,misc]{title}{%
  \mkbibquote{\usebibmacro{title:doiurl}{#1}}%
}

\DeclareFieldFormat[book]{title}{\mkbibemph{#1}}

\usepackage{amsmath}
\usepackage{bm}
\usepackage{bbm}
\usepackage[overload]{textcase}  
\usepackage{fontawesome5}  
\usepackage{mathtools}
\newcommand{\state}{{s}}
\newcommand{\ctrl}{{a}}

\newcommand{\ctrlset}{{\mathcal{\MakeTextUppercase{\ctrl}}}}

\DeclareMathOperator*{\expect}{\mathbb{E}}
\newcommand{\E}{\expect}

\newcommand{\s}{~\text{s}}

\usepackage{xcolor}
\definecolor{acqua}{rgb}{0.26,0.58,0.97}
\definecolor{orange}{RGB}{232,119,34}
\definecolor{softwhite}{rgb}{0.9,0.9,0.9}

\makeatletter
\renewcommand{\mathcolor}[2]{%
  \begingroup
  \colorlet{out}{.}\color{#1}#2\@ifnextchar_{\do@mathcolorsub}{\endgroup}%
}
\newcommand{\do@mathcolorsub}[2]{%
  _{{\color{out}#2}}\@ifnextchar^{\do@mathcolorsup}{\endgroup}%
}
\newcommand{\do@mathcolorsup}[2]{%
  ^{\color{out}#2}\endgroup
}
\makeatother

\newcommand{\side}{\leftarrow}
\newcommand{\up}{\uparrow}
\usepackage{tikz}
\usepackage{scalerel} 

\newcommand{\tetrisblock}[1]{%
  \mskip 2mu
  \scalerel*{%
    \begin{tikzpicture}[baseline=(base)]
      \path (0,0) -- (0,3); 
      \coordinate (base) at (0,0.5);
      #1
    \end{tikzpicture}%
  }{X}%
  \mskip 2mu
}

\newcommand{\tetrisL}{\tetrisblock{
    \foreach \c in {(0,0), (0,1), (0,2), (1,0)} {
        \draw[fill=orange!100, line width=0.2pt] \c rectangle ++(1,1);
    }
}}

\newcommand{\tetrisT}{\tetrisblock{
    \foreach \c in {(0,0), (1,0), (2,0), (1,1)} {
        \draw[fill=acqua!100, line width=0.2pt] \c rectangle ++(1,1);
    }
}}

\newcommand{\tetrisOne}{\tetrisblock{
        \draw[fill=gray!25, line width=0.6pt] (0,0) rectangle ++(1,1);
}}

\newcommand{\tetrisTwoStack}{\tetrisblock{
    \foreach \c in {(0,0), (0,1)} {
        \draw[fill=gray!25, line width=0.6pt] \c rectangle ++(1,1);
    }
}}

\newcommand{\tetrisTwoSide}{\tetrisblock{
    \foreach \c in {(0,0), (1,0)} {
        \draw[fill=gray!25, line width=0.6pt] \c rectangle ++(1,1);
    }
}}

\usepackage{ifthen}
\newboolean{include-notes} \setboolean{include-notes}{true}
\newboolean{include-remove}\setboolean{include-remove}{true}
\newboolean{include-new}   \setboolean{include-new}{true}

\newcommand{\jaime}[1]{\ifthenelse{\boolean{include-notes}}{\textcolor{orange}{\textbf{Jaime:} #1}}{}}
\newcommand{\elle}[1]{\ifthenelse{\boolean{include-notes}}{\textcolor{magenta}{\textbf{Elle:} #1}}{}}
\newcommand{\david}[1]{\ifthenelse{\boolean{include-notes}}{\textcolor{Cerulean}{\textbf{David:} #1}}{}}
\newcommand{\remove}[1]{\ifthenelse{\boolean{include-remove}}{\textcolor{red}{\sout{#1}}}{}}
\newcommand{\new}[1]{\ifthenelse{\boolean{include-new}}{\textcolor{blue}{#1}}{#1}}
\newcommand{\todo}[1]{\ifthenelse{\boolean{include-notes}}{\textcolor{blue}{\textbf{TODO:} #1}}{}}

\newcommand{\Pgoal}  
    {\tetrisT}
\newcommand{\Lgoal}
    {\tetrisL}
\newcommand{\PartialUp}
    {\tetrisTwoStack}
\newcommand{\PartialSide}
    {\tetrisTwoSide}
\newcommand{\PartialOne}
    {\tetrisOne}
\usepackage[acronym]{glossaries}
\makeglossaries

\newacronym{ai}{AI}{artificial intelligence}
\newacronym{rlhf}{RLHF}{Reinforcement Learning from Human Feedback}
\newacronym{ppva}{PPVA}{Pragmatic--Pedagogic Value Alignment}
\newacronym{ioc}{IOC}{inverse optimal control}
\newacronym{cirl}{CIRL}{cooperative inverse reinforcement learning}
\newacronym{pomdp}{POMDP}{partially observable Markov decision process}
\newacronym{mpc}{MPC}{model predictive control}
\newacronym{irl}{IRL}{inverse reinforcement learning}
\newacronym{tom}{ToM}{theory of mind}
\newacronym{rl}{RL}{reinforcement learning}
\newacronym{ppbr}{PPBR}{Pragmatic--Pedagogic Best Response}

\usepackage[colorlinks=true]{hyperref}
\usepackage{cleveref}

\begin{document}
\title{
Corrigible Assistance in One Round: Pragmatic--Pedagogic Best Response}
%
%
\author{Elle Lazarski\inst{1} \and Jaime Fernández Fisac\inst{1}}
\authorrunning{E. Lazarski et al.}
%
\institute{Department of Electrical and Computer Engineering\\Princeton University, USA}
\maketitle              
\begin{abstract}

Assistance games formalize human--robot collaboration under asymmetric information: the human knows the goal, while the robot must infer it from observation and interaction in order to assist effectively. In general, computing optimal assistance game strategies online is intractable, since exact solutions require planning in a POMDP. We identify a class of assistance games in which pragmatic--pedagogic reasoning resolves goal uncertainty in a single time step, rendering the full-horizon game exactly solvable by a tractable best-response procedure. Within this class, we show that mainstream inverse optimal control exhibits an inference ceiling that hinders alignment, while pragmatic--pedagogic reasoning overcomes this barrier by immediately disambiguating goals through actions that look equivalent under task execution alone. Finally, we validate our theoretical results and proposed method on a simple collaborative block-building example.


\end{abstract}
\glsresetall

\section{Introduction}
\label{sec:introduction}

As robots become more capable and are deployed in increasingly varied contexts, they must infer and adapt to their users' needs online rather than execute prespecified routines. In \gls{ai}, conventional alignment pipelines like \gls{rlhf} optimize for human approval of generated outputs, such as through thumbs-up or A/B preference feedback~\cite{christiano2017deep}. However, this signal is known to be an imperfect, often problematic proxy, since it tends to neglect the downstream effects of individual decisions~\cite{casper2023open,lang2024deception,williams2025targeted}. In human--robot interaction, the coupling between robot operation and human behavior over time makes it all the more necessary to approach alignment with respect to long-term outcomes instead of isolated robot actions~\cite{bestick2017implicitly, liu2016goal, bobu2018learning, bobu2024aligning}.

To capture the temporal dimension of alignment,
\textit{assistance games} \cite{fern2014decision,hadfieldmenell2016cooperative} 
pose the problem as a dynamic two-player collaboration between a human user and a robot assistant, who aims to help realize the human's objective but is \emph{uncertain} about what it is.
%
The resulting equilibrium solutions have been shown to present desirable properties, which extend the notion of 
optimal information seeking to the two-player setting.
In particular,
the human’s actions often carry out a \textit{pedagogic} function, strategically conveying actionable information about the goal,
while the robot's responses are \textit{pragmatic}, interpreting human cues as purposefully (rather than circumstantially) communicative~\cite{fisac2017pragmatic,malik2018efficient},
consistent with modern cognitive science accounts of human teaching and learning~\cite{shafto2014rational}.
Unfortunately, despite their theoretical strengths,
assistance game solutions are largely considered computationally intractable, due to the need to plan in the robot's information space~\cite{laidlaw2025assistancezero}.

In this work, we identify a special---but nontrivial---class of assistance games in which pragmatic--pedagogic solutions are simultaneously \emph{highly effective} and \emph{easily computable}, fully disambiguating the human’s goal in a single time step.
Consequently, for this class of problems, an efficient short-horizon approximation (analogous to single-player QMDP~\cite{littman1995learning}) yields an optimal strategy pair for the full-horizon game.
%
We further show that this equilibrium can be readily obtained
in a single round of player best responses starting from a mainstream \gls{ioc}---or \gls{irl}---solution.
Crucially, while \gls{ioc} often exhibits an \emph{inference ceiling} that impedes one-step alignment,
the pragmatic--pedagogic solution fully overcomes this limitation, uniquely disambiguating goals through actions that appear equivalent from the standpoint of task execution alone.
%
As a result, the pragmatic robot strategy is \mbox{\emph{maximally empowering}}, affording the human the option to convey and achieve any goal regardless of the robot's initial belief,
and rendering the robot's operation \mbox{\emph{corrigible}},
a key requirement for robust alignment and long-term safety~\cite{wiener1960moral}.

Our contributions in this work can be summarized as follows:
\begin{enumerate}
    \item \textbf{One-round convergence to a pragmatic–pedagogic equilibrium.}
    We quantify the \gls{ioc} inference ceiling in terms of a posterior belief bound and show that pragmatic--pedagogic reasoning can surpass it. We formalize the conditions that enable goal disambiguation in one time step and prove that one round of best responses reaches an optimal pragmatic--pedagogic equilibrium for the class of \textit{action-separable} assistance games.

    \item \textbf{A practical methodology for tractable pragmatic assistance.}
    We propose \gls{ppbr}, an algorithm that directly computes the equilibrium solution to an action-separable assistance game, yielding a human-empowering robot strategy.

    \item \textbf{Empirical validation.}
    We validate our theoretical results empirically in a collaborative block-building domain, comparing \gls{ioc} against \gls{ppbr} under a comprehensive range of initial conditions.
\end{enumerate}

\section{Related Work}
\label{sec:prelims_related}

Traditional
\gls{ioc}/\gls{irl} 
aims to recover the objective pursued by a human ``expert'' by observing demonstrations of their behavior, assumed to take place in isolation~\cite{ng2000algorithms,ramachandran2007bayesian,ziebart2008maximum}. 
In interactive settings, however, 
the robot is not a passive observer but an active player whose own actions may affect the human's outcome.
The human may therefore choose actions not only to directly generate utility but also to indirectly improve expected outcomes by influencing the robot.

Assistance games, initially introduced under the name \gls{cirl}, formalize this idea as a cooperative game with asymmetric information \cite{fern2014decision,hadfieldmenell2016cooperative}. In particular, the human knows the true goal, while the robot must infer it from observation and interaction. Solving a \gls{cirl} game can be reduced to solving a \gls{pomdp}, in which the robot's belief over goals serves as a sufficient statistic for optimal decision-making \cite{hadfieldmenell2016cooperative}. However, this reduction inherits the computational burden of \gls{pomdp} planning,
barring its use in practical problems.

Subsequent research in assistance games showed that optimal human--robot strategy pairs must satisfy a specific dynamic programming relation expressed as a fixed point in a \textit{pragmatic--pedagogic Bellman recursion},
whereby the human chooses actions pedagogically to convey strategically relevant information to a suitably attuned robot partner,
and the robot in turn interprets human actions pragmatically, treating them as cues purposefully chosen to convey information~\cite{fisac2017pragmatic}.
Despite an exponential complexity improvement over the naïve POMDP reduction, the resulting belief-space planning is still intractable for runtime computation of assistance strategies.

Here, we build on the theoretical insights established by prior assistance game efforts to investigate runtime-computable robot strategies that preserve cooperative structure while enabling online assistance.
We combine classical short-horizon approximation ideas from decision theory~\cite{littman1995learning} and truncated iterated-best-response approximations from behavioral game theory~\cite{stahl1995players,costa2001cognition,camerer2004cognitive},
and show that they allow us to exactly solve the full-horizon assistance game for a class of problems in which
the robot's uncertainty can be resolved in a single time step.

Finally, some work in the \gls{ai} alignment literature has brought into question the usefulness of pragmatic robot behavior, due to potentially higher sensitivity to modeling assumptions (in particular, whether or not the human user intends to behave pedagogically)~\cite{milli2020literal}. 
Our results shed new light on the matter, suggesting an alternative perspective:
regardless of modeling accuracy,
robot strategies derived from pragmatic--pedagogic solutions make the robot comparatively more responsive to human action cues, increasing the human's effective controllability over outcomes and mitigating known overconfidence issues with non-pragmatic (sometimes called ``literal'') goal-inferring robots (typically \gls{ioc}) \cite{milli2017should,hadfield2017off,bobu2018learning}. 

\section{Problem Formulation: Assistance Games}
\label{sec:problemformulation}

We study assistance games $\mathcal{G}$ in which a human and a robot take turns acting in a shared environment to achieve the human's objective. Critically, this objective is known to the human but, a priori, unknown to the robot. Let
\[
\mathcal{G} = \big\langle \mathcal{S}, \{\mathcal{A}^H, \mathcal{A}^R\}, T_s, \{\Theta, R\}, P_0, \gamma \big\rangle,
\]
where $\mathcal{S}$ denotes the space of world states; $\mathcal{A}^H$, $\mathcal{A}^R$ are the human and robot action spaces; $\Theta$ is a set of possible goal parameters encoding the human's objective; $T_s(s_{t+1} \mid s_t, a_t^H, a_t^R)$ is the transition probability measure; $R(s_t, a_t^H, a_t^R; \theta)$ is the shared reward function, parameterized by $\theta \in \Theta$ (whose value is only observed by the human); $P_0(s_0, \theta)$ is the initial joint distribution over states and goals; and $\gamma \in [0,1]$ is a discount factor. At each time step $t$, the human selects $a_t^H \in \mathcal{A}^H$ first, after which the robot observes $a_t^H$ and selects $a_t^R \in \mathcal{A}^R$; the joint action $(a_t^H,a_t^R)$ then induces a successor state through $T_s$.

The robot maintains a Bayesian posterior belief~$b_t^+\in\Delta(\Theta)$ over candidate goal hypotheses, which is updated after each observed $a^H_t$.
Throughout the paper, we let $b_t$ denote the robot's prior belief at the start of time step $t$, and we use $b_t^+$ for the robot's posterior \emph{after} observing the human's action, specifying the robot's belief update model (e.g., $b_t^{R1+}$) when relevant.

Following the assistance game literature, $P_0$ is known to both players, which means that $b_t^+$ can be computed by \emph{either} player as a sufficient statistic of the robot's information state after observing $(s_0, a^H_0, a^R_0, \dots, s_t, a^H_t)$ \textit{under any particular human policy}. 
Since the robot's belief may inform its behavior, it is \textit{also} strategically relevant to the human.
Therefore, the human's optimal assistance game strategy will in general be a stochastic policy
$\pi^H(a_t^H\mid s_t,b_t; \theta)$,
whereas the robot---not privy to~$\theta$, but observing $a_t^H$ before acting---must choose a policy
$\pi^R(a_t^R\mid s_t,b_t^+, a_t^H)$.
Solving the assistance game amounts to finding a team strategy $\pi := (\pi^H, \pi^R)$ that
maximizes the expected time-discounted return
\begin{equation*}
J(\pi^H, \pi^R)
:=
\expect_{\substack{(s_0,\theta) \sim P_0\\\tau \sim (\pi, T)}}
\left[\sum_{t=0}^{\infty} \gamma^t R(s_t, a_t^H, a_t^R; \theta)\right],
\end{equation*}
where $\tau := (s_0, a^H_0, a^R_0, s_1, a^H_1, a^R_1, \dots)$ denotes the gameplay trajectory, whose distribution is sequentially induced by team strategy $\pi$ and 
$(s_t,b_t)$-transitions $T := (T_s, T_b)$,
with $(s_0, \theta)$ drawn jointly from $P_0$, and $b_0$ defined as the conditional distribution on $\theta$ given $s_0$.
We omit time indices when clear from context.

The dependence of $T_b$ on $\pi^H$ is
a distinctive feature of assistance games:
belief transitions depend not only on the observed human action but, indirectly, on the entire human policy, which determines the observation \emph{likelihood} model~\cite{fisac2017pragmatic}.

\paragraph{Running example:}
We consider a collaborative block-building task in which the human's latent goal $\theta^\star \in \{\Pgoal, \Lgoal\}$ specifies a target Tetris shape.
The human is indifferent to 
where the structure is built and what direction it faces, as long as it is upright.
To avoid known action-multiplicity artifacts in Boltzmann likelihood models~\cite{bobu2020less}, we represent states and actions in the quotient spaces induced by the problem's underlying symmetries: states that differ only by 
an \(\mathrm{SE}(2)\) transform 
(translation and rotation on the ground plane) 
are treated as equivalent, as are actions that induce equivalent state changes.
In our two-goal example, the quotient action space $\ctrlset(\state)$ in most intermediate states of interest ($\state = \PartialOne, \PartialSide,\PartialUp,\dots$) contains two block placements 
(\textsc{side}/$\side$ and \textsc{up}/$\up$).

\paragraph{Suboptimal play:} While other actions are possible in principle (e.g., placing a new block far away from the existing ones), we exclude them from our analysis because they are \emph{exponentially} less likely under all hypotheses.
A robot observing such actions would typically lose situational confidence and choose not to act~\cite{bobu2018learning}.
\looseness=-1

\section{Approach: One Time Step, One Best-Response Round}
\label{sec:approach}

\subsection{Alignment in One Time Step}
Extending the QMDP approximation in single-agent \glspl{pomdp} to the two-player setting, we decompose the assistance game's horizon into a present uncertain phase and a hypothesized later phase in which the robot has gained full certainty about the human's true goal. While this is, in general, an optimistic approximation of the problem (rarely tight for typical POMDPs), we will show in the next section that it is in fact exact in task settings with comparatively mild asymmetry conditions. Intuitively, if both players can independently arrive at an unambiguous correspondence between goals and actions, the robot \textit{will} become fully confident in the human's true goal after a single time step, thereby rendering the one-step goal disambiguation assumption accurate.



\paragraph{Perfect-Information Subgame.}
We can readily compute the oracle value $V^\mathrm{O}(s; \theta)$ that characterizes the fully observed phase of planning. With the robot also privy to the human's true goal, the assistance game reduces to a single-player MDP over the joint action space $\mathcal{A}^H \times \mathcal{A}^R$, where the two players act as a centralized ``hive mind.'' The goal-parameterized state--action value for the team is
\begin{equation}
\label{eq:q_team}
Q^\mathrm{O}(s,a^H,a^R;\theta) := r(s,a^H,a^R;\theta) + \gamma\,\expect\!\big[V^\mathrm{O}(s';\theta)\big],
\end{equation}
with $s'\sim T_s(\cdot\mid s,a^H,a^R)$.

\subsection{Equilibrium in One Round}
We operationalize pragmatic--pedagogic inference from the perspective of a robot helper that reasons about human behavior using a truncated hierarchy of \emph{hypothesized} player models. Specifically, the level-3 robot, denoted R3, simulates H0, R1, and H2 as follows. 


    

\noindent\textbf{\textit{H0: Solipsistic Human Expert.}} Under a given hypothesis $\theta$, H0 acts solely to optimize for that objective without considering how their actions influence the robot’s belief or behavior. Hence, H0 serves as the ``expert demonstrator'' in classical \gls{ioc}/\gls{irl} methods.

Let $V^{H0}(s;\theta)$ represent the maximum expected reward-to-go from state $s$ when the human works alone. This single-player MDP over $\mathcal{A}^H$ yields the solipsistic state--action value $Q^{H0}(s,a^H;\theta) := r(s,a^H;\theta) + \gamma\,\expect\!\big[V^{H0}(s';\theta)\big]$, with $s'\sim T_s(\cdot\mid s,a^H)$. Then
\begin{equation}
\label{eq:pi_h0}
\pi^{H0}(a^H \mid s; \theta) \propto \exp\!\big(\beta^H\,Q^{H0}(s, a^H; \theta)\big),
\end{equation}
where $\beta^{H} > 0$ is an inverse-temperature or ``rationality'' parameter controlling how strongly H0 favors high-value actions under $\theta$:
as $\beta^{H} \to 0$, the policy approaches a uniform distribution over all actions;
as $\beta^{H} \to \infty$, the policy concentrates on optimal actions, converging to a uniform distribution over only the $\arg\max$ set.

\noindent\textbf{\textit{R1: Naïve Robot Learner.}} R1 assumes the human behaves as H0 and updates its belief $b$ by Bayes' rule under likelihood $\pi^{H0}(a^H\mid s;\theta)$, yielding $b^{R1+}(\cdot\mid a^H)$.

Define $Q^{R1}(s, b^{R1+}, a^H, a^R) := \E_{\theta \sim b^{R1+}(\cdot\mid a^H)}\!\big[Q^\mathrm{O}(s,a^H,a^R;\theta)\big]$,
and
\begin{equation}
\label{eq:pi_r1}
\pi^{R1}(a^R \mid s, b^{R1+}, a^H)
\propto 
\exp\!\big(\beta^R\,Q^{R1}(s, b^{R1+}, a^H, a^R)\big),
\end{equation}
where $\beta^R > 0$ is the robot's inverse-temperature parameter.
Note that the robot's action is
informed by its posterior belief after observing the human's action.

\noindent\textbf{\textit{H2: Pedagogic Human User.}} For each candidate human action $a^H$, H2 anticipates R1's belief update and subsequent behavior, evaluating $a^H$ under a given hypothesis $\theta$ by taking an expectation over the induced R1 response distribution. Thus, $a^H$ is valuable to H2 both insofar as it directly advances $\theta$ and insofar as it elicits better robot follow-up under $\theta$. H2's choice of action is therefore \textit{strategic} rather than purely task-directed.

Define $Q^{H2}(s,b,a^H;\theta)
:=\expect_{a^R\sim \pi^{R1}(\cdot\mid s,b^{R1+},a^H)}\!\big[Q^{\text{O}}(s,a^H,a^R;\theta)\big]$,
and
\begin{equation}
\label{eq:pi_h2}
\pi^{H2}(a^H\mid s,b;\theta)
\propto
\exp\!\big(\beta^H\,Q^{H2}(s,b,a^H;\theta)\big).
\end{equation}

\noindent\textbf{\textit{R3: Pragmatic Robot Helper.}} After observing $a^H$, R3 updates its belief $b$ by Bayes' rule using the H2 likelihood model $\pi^{H2}(a^H\mid s,b;\theta)$, yielding $b^{R3+}(\cdot \mid a^H)$. It then selects its response---which is executed in the real environment---by maximizing expected value under the posterior:
\begin{equation}
\label{eq:pi_r3}
a^{R3\star}:=\arg\max_{a^R\in\mathcal A^R(s)}\expect_{\theta\sim b^{R3+}(\cdot \mid a^H)}\!\big[Q^{\text{O}}(s,a^H,a^R;\theta)\big].
\end{equation}

We will later show that, for the problem class we consider, truncating this level-$k$ hierarchy at R3 already reaches an optimal pragmatic--pedagogic equilibrium. Hence, there is no need to iterate further levels (H4, R5, and so on).

\section{Analysis: Pragmatic--Pedagogic One-Step Alignment}

In this section, we establish important properties of the R3--H2 solution and the conditions under which it allows us to recover an optimal equilibrium of the full-horizon assistance game.

\subsection{IOC Inference Ceiling and Incorrigibility}

We begin by examining the fundamental alignment limitations of \gls{ioc}/\gls{irl} in the assistance context. Let
\begin{equation*}
\mathcal{A}^{H0\star}(s;\theta):=\arg\max_{a^H\in\mathcal{A}^H(s)} Q^{H0}(s,a^H;\theta)
\end{equation*}
denote the set of H0-optimal actions under goal $\theta$ in state $s$. In general, an \gls{ioc} inference ceiling appears whenever one goal's H0-optimal action set is a \emph{subset} of the H0-optimal action set for another goal.

\subsubsection{Two Candidate Goals}

We first formalize this pathology in the simplest possible setting with only two candidate goal hypotheses.

\begin{lemma}[Two-goal IOC inference ceiling]
\label{lem:ioc_ceiling}
    Let $\Theta = \{\theta^\star, \theta'\}$, and suppose that $a^\dagger \in \mathcal{A}^H(s)$ is H0-optimal for both goals, with $a^\dagger \in \mathcal{A}^{H0\star}(s;\theta^\star) \subseteq \mathcal{A}^{H0\star}(s;\theta')$.
    Assume further that $|\mathcal{A}^{H0\star}(s;\theta^\star)| = k$ and $|\mathcal{A}^{H0\star}(s;\theta')| = m$, where $1 \le k \le m$.
    Then, R1's posterior belief in $\theta^\star$ after observing $a^\dagger$ from conditions $(s, b)$ converges, as $\beta^H \to \infty$, to
    \begin{equation}\label{eq:ioc_ceiling_two_goal}
    b^{R1+}_\infty(\theta^\star \mid a^\dagger) :=
    \lim_{\beta^H\to\infty} b^{R1+}(\theta^\star\mid a^\dagger) =
    \frac{m\,b(\theta^\star)}
    {m\,b(\theta^\star) + k\,b(\theta')}\,,
    \end{equation}
    which bounds $b^{R1+}_\infty$ strictly below 1.
\end{lemma}
\begin{proof}
As $\beta^H\to\infty$, $\pi^{H0}(a^\dagger\mid s;\theta^\star)\to \frac{1}{k}$ and $\pi^{H0}(a^\dagger\mid s;\theta')\to \frac{1}{m}$. Substituting these limiting likelihoods into the R1 Bayes update and normalizing gives~\eqref{eq:ioc_ceiling_two_goal}.
\qed
\end{proof}
\begin{remark}\label{rem:no_new_info}
If the two goals assign equal H0 likelihood to $a^\dagger$, R1 learns no new information to disambiguate between $\theta^\star$ and $\theta'$. In particular, when the optimal action sets are identical, \eqref{eq:ioc_ceiling_two_goal} leaves the prior unaltered, yielding $\lim_{\beta^H\to\infty} b^{R1+}(\theta^\star\mid a^\dagger) = b(\theta^\star)$.
\end{remark}

\paragraph{Running example:} The human's first action---placing the anchor block---is necessarily uninformative to the robot. Since by assumption the human is indifferent to position and north--south--east--west orientation, all feasible initial block placements belong to the same equivalence class, and placing \textit{a} block is the unique optimal action for any $\theta\in\{\Pgoal, \Lgoal\}$ (cf.~\Cref{rem:no_new_info}).

The robot can select any follow-up action $a^R\in\{\side,\up\}$ arbitrarily regardless of its prior belief, since both have the same expected value under the two goal hypotheses (both goals include a block above and adjacent to the anchor block). For our subsequent analysis, we assume the robot places \textit{above} the anchor, i.e., $a^R=\up$. The alternate case is symmetrical in structure. We fix the resulting partial state $s=\PartialUp$ as the reference state for the remainder of the paper.

From \Cref{lem:ioc_ceiling}, with $k=1$ (since $\side$ is uniquely H0-optimal under $\Pgoal$) and $m=2$ (since $\side$ and $\up$ are \textit{both} H0-optimal under $\Lgoal$), the \gls{ioc} ceiling is
\begin{equation}
\label{eq:ioc_ceiling_tetris}
b^{R1+}_\infty(\Pgoal \mid \side) = \frac{2b(\Pgoal)}{1+b(\Pgoal)}.
\end{equation}
When $b(\Pgoal)<\frac{1}{3}$, this ceiling lies below the robot’s one-step decision boundary (which is $b=\frac{1}{2}$). Importantly, no amount of assumed human rationality ($\beta^H \gg 1$) can push the one-step \gls{ioc} posterior beyond $\frac{1}{2}$. This means that, if the human's true goal is $\theta^\star = \Pgoal$ and the robot starts off placing more than $\frac{2}{3}$ belief on $\Lgoal$, there is no course of action available to the human to prevent the robot from building the wrong structure.

Specifically, if the human places a block on top of the partial structure ($a^H = \up$), the \gls{ioc} robot's optimal response will place a block on the side ($a^R = \side$); and if the human places a block on the side ($a^H = \side$), the \gls{ioc} robot will confidently follow up with a final block on top ($a^R = \up$); either scenario results in completing the $\Lgoal$ structure. In other words, even in simple assistance games, the \gls{ioc} robot strategy is short-term \emph{incorrigible} from seemingly benign initial conditions.

\subsubsection{Multi-Goal Assistance Games}

We now express the inference ceiling in complete generality.

\begin{lemma}[General IOC inference ceiling]
\label{lem:ioc_ceiling_multi}
Let $a^\dagger \in \mathcal{A}^H(s)$, and assume that there exists some candidate goal $\bar\theta\in\Theta$ such that $a^\dagger\in
\mathcal{A}^{H0\star}(s;\bar\theta)$. Then, for every $\theta\in\Theta$, R1's posterior belief after observing
$a^\dagger$ from conditions $(s, b)$ converges, as $\beta^H\to\infty$, to
\begin{equation}
\label{eq:ioc_ceiling_multi}
b^{R1+}_\infty(\theta \mid a^\dagger)
:= \lim_{\beta^H \to \infty} b^{R1+}(\theta \mid a^\dagger)
= \frac{\dfrac{b(\theta)}{|\mathcal{A}^{H0\star}(s;\theta)|}\,\mathbbm{1}\!\big[a^\dagger\in\mathcal{A}^{H0\star}(s;\theta)\big]}{\displaystyle\sum_{\theta'\in\Theta} \dfrac{b(\theta')}{|\mathcal{A}^{H0\star}(s;\theta')|}\,\mathbbm{1}\!\big[a^\dagger\in\mathcal{A}^{H0\star}(s;\theta')\big]}.
\end{equation}
In particular, $b^{R1+}_\infty < 1$ whenever $a^\dagger$ is H0-optimal under $\theta$ and at least one alternative $\theta' \neq \theta$.
\end{lemma}
\begin{proof}
As $\beta^H\to\infty$, the H0 softmax policy converges to the uniform distribution over the $\arg\max$ action set:
\[
\pi^{H0}(a^\dagger\mid s;\theta)
\;\to\;
\frac{1}{|\mathcal{A}^{H0\star}(s;\theta)|}\,
\mathbbm{1}\!\big[a^\dagger\in\mathcal{A}^{H0\star}(s;\theta)\big].
\]
Substituting into the R1 Bayes update and normalizing gives~\eqref{eq:ioc_ceiling_multi}.
\qed
\end{proof}

\subsection{Pedagogic Leverage and Corrigibility}

We next show that pragmatic--pedagogic inference can break the belief ceiling experienced by the \gls{ioc} robot R1. Since H2 evaluates actions by anticipating R1’s response, actions that are equivalent from H0's solipsistic perspective need not remain equivalent under H2’s reasoning.


\begin{definition}[Level-$k$ advantage]
\label{def:ped_advantage}
Let $Hk$ be a human decision model with state--action value
$Q^{Hk}(s,b,a^H;\theta)$. For any two human actions $a, \tilde a \in\mathcal A^H(s)$, the level-$k$ advantage of $a$ against $\tilde a$ under goal $\theta$ from conditions $(s,b)$ is the difference in their expected value:
\begin{equation*}
\label{eq:advantage}
A^{Hk}_\theta(a,\tilde a;s,b) := Q^{Hk}(s,b,a;\theta) - Q^{Hk}(s,b,\tilde a;\theta).
\end{equation*}
\end{definition}

The level-$k$ advantage measures how strongly H$k$ prefers $a$ relative to $\tilde a$ under a given hypothesis $\theta$. Note that \Cref{def:ped_advantage} can be applied to $Q^{H0}$ by simply ignoring $b$, since the solipsistic state--action value does not depend on the robot's belief. When $k=2$, we refer to this quantity as the \textit{pedagogic advantage}.

We emphasize that, in the Boltzmann-rational setting, $Q^{H2}$ is a function of both players' rationality parameters: $\beta^H$ influences R1's belief update after observing $a^H$, and $\beta^R$ determines its subsequent response probabilities. This is in contrast with $Q^{H0}$, which depends on neither $\beta^H$ nor $\beta^R$. Throughout the paper, we compute $Q^{H2}$ in the rational human limit ($\beta^H \to \infty$).

Focusing on H2, we now characterize when human actions serve as unambiguous goal-identifying signals. In particular, every action that is H2-optimal under goal $\theta$ must be strictly suboptimal under all $\theta'\neq\theta$. Equivalently, the H2-optimal action sets for distinct goals must be \textit{disjoint}. Checking this condition for each candidate goal yields the pedagogic leverage map, defined as follows.

\begin{definition}[Pedagogic leverage map]
\label{def:ped_leverage}
The pedagogic leverage map from conditions $(s,b)$ is the set-valued map
\begin{equation*}
\mathcal L_{s,b}^{H2}: \, \Theta\rightrightarrows\mathcal A^H(s)
\end{equation*}
defined, for each $\theta\in\Theta$, by
\begin{equation*}
\mathcal L_{s,b}^{H2}(\theta)
:=
\begin{cases}
\mathcal A^{H2\star}(s;\theta),
&
\text{if }
\mathcal A^{H2\star}(s;\theta)
\cap
\mathcal A^{H2\star}(s;\theta')
=
\emptyset
\quad
\forall\,\theta'\neq\theta,
\\[0.5em]
\emptyset,
&
\text{otherwise},
\end{cases}
\end{equation*}
where $\mathcal A^{H2\star}(s;\theta)
:=
\arg\max_{a^H\in\mathcal A^H(s)}
Q^{H2}(s,b,a^H;\theta)$.
\end{definition}

We now show that whenever an observed action has pedagogic leverage toward a particular goal, R3's posterior concentrates on that goal as $\beta^H \to \infty$.

\begin{proposition}[Pedagogic leverage enables full one-step alignment]
\label{prop:leverage_full_alignment}
Fix any $\beta^R>0$, and suppose that action $a^H\in\mathcal L_{s,b}^{H2}(\theta)$. Then,
\begin{equation}
\label{eq:prop1}
b^{R3+}_\infty(\theta \mid a^H) := \lim_{\beta^H\to\infty} b^{R3+}(\theta \mid a^H)=1.
\end{equation}
\end{proposition}
\begin{proof}
By \Cref{def:ped_leverage}, $a^H\in\mathcal A^{H2\star}(s;\theta)$ and $a^H\notin\mathcal A^{H2\star}(s;\theta')$ for all $\theta'\neq\theta$, so
\[
\pi^{H2}(a^H\mid s,b;\theta)
\to
\frac{1}{|\mathcal A^{H2\star}(s;\theta)|}
>0,
\qquad
\pi^{H2}(a^H\mid s,b;\theta')
\to 0
\quad
\forall\,\theta'\neq\theta,
\]
as $\beta^H\to\infty$. Substituting these limiting likelihoods into
the R3 Bayes update and normalizing gives \eqref{eq:prop1}.
\qed
\end{proof}

\begin{corollary}[Pedagogic leverage breaks the IOC ceiling]
\label{cor:breaks_ioc_ceiling}
If Lemma~\ref{lem:ioc_ceiling} holds for $\theta^\star,\theta'$ and action $a^\dagger\in \mathcal L_{s,b}^{H2}(\theta^\star)$, then
\begin{equation*}
b^{R1+}_\infty(\theta^\star\mid a^\dagger)
=\frac{m\,b(\theta^\star)}{m \,b(\theta^\star)+k\,b(\theta')}
\quad \text{while} \quad b^{R3+}_\infty(\theta^\star\mid a^\dagger)=1.
\end{equation*}
\end{corollary}


\paragraph{Running example:}
Return to the two-goal Tetris example with $\theta\in\{\Pgoal,\Lgoal\}$ and $a^H,a^R\in\{\side,\up\}$. We once again analyze the robot's inference from reference state $s = \PartialUp$, after the first two blocks have been placed. We use an additive team reward: each correct placement gives $+1$, and each incorrect placement gives $-1$, so the joint reward takes values in $\{-2,0,+2\}$ depending on whether 0, 1, or 2 blocks are placed correctly. In the one-step setting ($\gamma=0$), $Q^\text{O}(s,a^H,a^R;\theta)=R_\theta(a^H,a^R)$ is the immediate joint reward, given by
\[
R_{\Pgoal}=
\begin{bmatrix}
2 & \quad 0\\
0 & \quad -2
\end{bmatrix},
\qquad
R_{\Lgoal}=
\begin{bmatrix}
0 & \quad 2\\
2 & \quad 0
\end{bmatrix},
\]
where rows denote $a^H\in\{\side,\up\}$ and columns denote $a^R\in\{\side,\up\}$.

Fix any prior $b(\Pgoal)\in(0,1)$, and let the observed human action be $a^H=\side$, which results in the IOC inference ceiling~\eqref{eq:ioc_ceiling_tetris}.
Direct evaluation of $Q^{H2}$ 
yields the $\arg \max$ action sets $\mathcal A^{H2\star}(s;\Pgoal)=\{\side\}$ and $\mathcal A^{H2\star}(s;\Lgoal)=\{\up\}$ as $\beta^H \to \infty$ for all $\beta^R > 0$.
Therefore, by \Cref{def:ped_leverage}, $a^H = \side$ acquires \textit{pedagogic leverage} toward $\Pgoal$: $\mathcal L_{s,b}^{H2}(\Pgoal)=\{\side\}$. Likewise, $\mathcal L_{s,b}^{H2}(\Lgoal)=\{\up\}$.
Hence, by
\Cref{cor:breaks_ioc_ceiling},
\[
b^{R3+}_\infty(\Pgoal\mid\side)=1.
\]
Crucially, even in the prior regime $b(\Pgoal)<\tfrac{1}{3}$, the pragmatic (R3) robot's optimal response will place a block on the side ($a^R=\side$) as $\beta^H \to \infty$, completing $\Pgoal$. R3 is thus one-step \emph{corrigible} from any initial conditions. The emergent convention of $a^H=\side$ signaling $\Pgoal$ (and $a^H=\up$ signaling $\Lgoal$) can be viewed as a Schelling focal point.\footnote{Coordination problems with multiple equilibria often admit \textit{Schelling focal points}~\cite{schelling1957bargaining,schelling1980strategy}, or prominent solutions that players gravitate toward based on shared intuition or salience---for example, choosing ``12:00 PM'' as a default meeting time. Here, the solipsistic (H0) human serves as the implicit saliency model that the human and robot use to break possible ambiguity.} See \Cref{fig:overlay}.

\begin{figure}[t]
    \centering
    \includegraphics[width=0.85\linewidth]{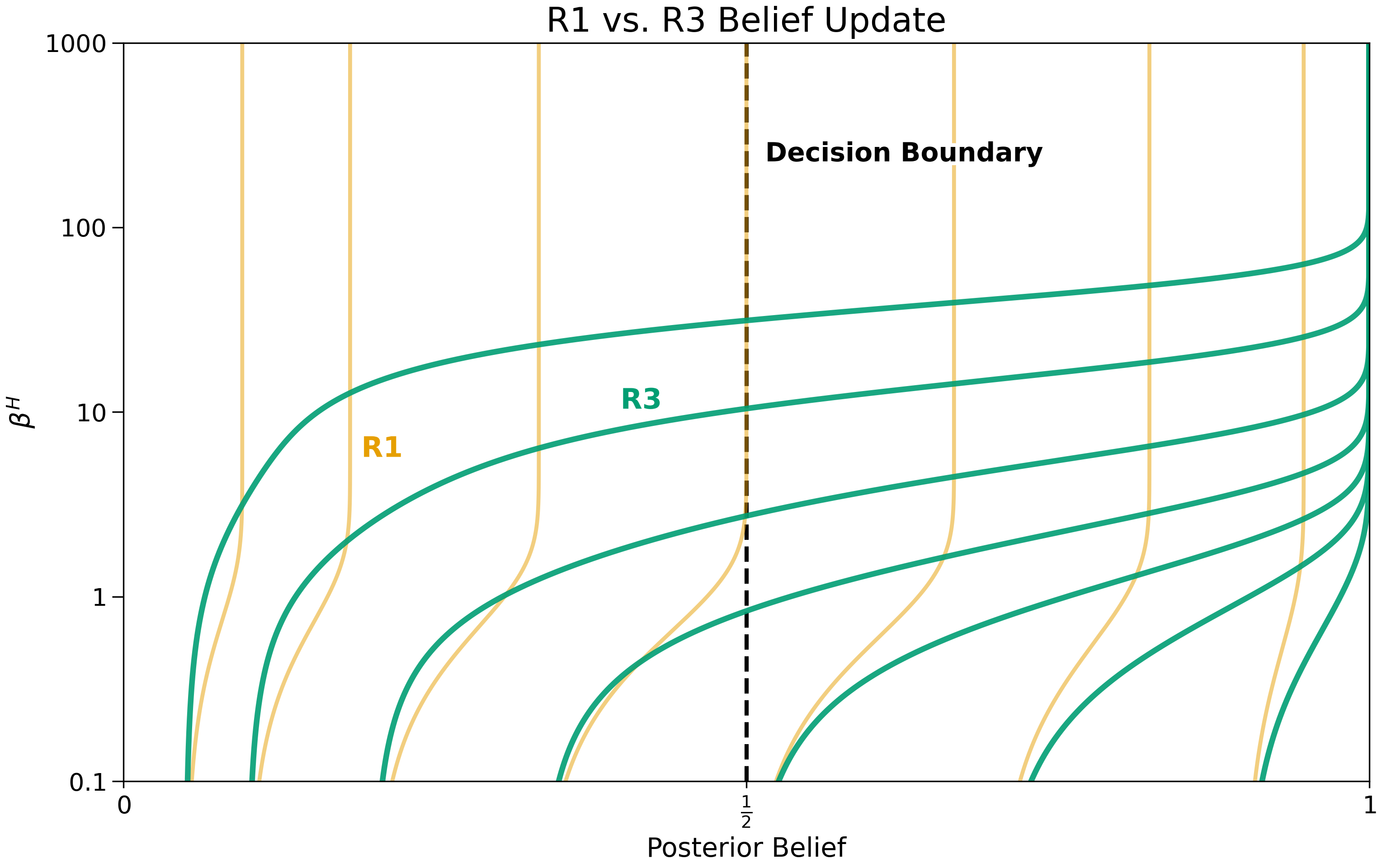}
    \caption{One-step posterior $b^+(\protect\Pgoal \mid  \protect\side)$ at $s = \protect\PartialUp$ as a function of human rationality $\beta^H$ for representative priors $b(\protect\Pgoal)\in(0,1)$. The \gls{ioc} update (R1, orange) saturates at the posterior ceiling of Lemma~\ref{lem:ioc_ceiling}, while the pragmatic update (R3, green; $\beta^R=1$) crosses the decision boundary $b=\tfrac{1}{2}$ (black dashed line) and approaches certainty in $\protect\Pgoal$ as $\beta^H$ increases.}
    \label{fig:overlay}
\end{figure}

\subsection{Vanishing Pedagogic Leverage and Scaling Behavior}
\label{sec:scaling}

We are primarily interested in the limiting behavior of rational agents R3--H2, which we note is obtained by having H2 reason about a \textit{noisily} rational R1. For finite $\beta^R > 0$, R1's softmax has full support, so \textit{all} robot responses---including follow-ups that are suboptimal under its posterior---contribute to H2's expected value. Consequently, even small posterior shifts (bounded by the IOC ceiling) can induce $Q^{H2}$ gradients that cause human actions to acquire pedagogic leverage (\Cref{def:ped_leverage}). This leverage might vanish if we simply plugged in R1's deterministic best response to each human action.

We study these effects in the two-action setting with $a, \tilde a\in\mathcal A^H(s)$ and target goal $\theta$, writing the log-odds of H2's softmax policy in terms of the pedagogic advantage (\Cref{def:ped_advantage}):
\begin{equation*}
\log\frac{\pi^{H2}(a\mid s,b;\theta)}{\pi^{H2}(\tilde a\mid s,b;\theta)}
=\beta^H A^{H2}_\theta(a, \tilde a; s,b).
\end{equation*}
Importantly, $A^{H2}_\theta$ depends on $\beta^R$ through R1's softmax policy, which is used to define $Q^{H2}$. If $A^{H2}_\theta$ vanishes in some regime of $\beta^R$, then driving $\pi^{H2}$ close to $1$ (for $A^{H2}_\theta > 0$) or to $0$ (for $A^{H2}_\theta < 0$) requires $\beta^H|A^{H2}_\theta| \to \infty$, i.e., $\beta^H$ must scale faster than the reciprocal rate.

We demonstrate this scaling behavior in our running example at finite $\beta^H, \beta^R$ values; see \Cref{fig:belief_heatmap_R3} and our analysis below. However, we emphasize that even a vanishing amount of leverage will be exploited with probability 1 by H2 in the limit of full rationality ($\beta^H \to \infty$).


\paragraph{Running example:}
At $s=\PartialUp$, let
$A_\theta:=A^{H2}_\theta(\side,\up;s,b)$,
$x:=\pi^{R1}(\side\mid s,b^{R1+},\side)$, and
$y:=\pi^{R1}(\up\mid s,b^{R1+},\up)$.
The reward matrices give
\begin{equation*}
\begin{aligned}
Q^{H2}(s,b,\side;\Pgoal)&=2x,
&
Q^{H2}(s,b,\up;\Pgoal)&=-2y,
\\
Q^{H2}(s,b,\side;\Lgoal)&=2(1-x),
&
Q^{H2}(s,b,\up;\Lgoal)&=2(1-y),
\end{aligned}
\end{equation*}
so $A_{\Pgoal}=2(x+y)$ and $-A_{\Lgoal}=2(x-y)$. $A_{\Pgoal}>0$ makes $a^H = \side$ advantageous under $\Pgoal$, while $-A_{\Lgoal}>0$ makes $a^H = \side$ disadvantageous under $\Lgoal$.

\textbf{Exponential regime
($b(\Pgoal)\in(0,\tfrac13)$, $\beta^R\to\infty$).} 
In the limit $\beta^H\to\infty$, R1's posterior after observing $a^H = \side$ converges to the IOC ceiling $p := b^{R1+}_\infty(\Pgoal\mid\side)$ (\ref{eq:ioc_ceiling_tetris}), and R1 assigns values $2p$ and $2(1-p)$ to follow-up actions $a^R=\side$ and $a^R=\up$, respectively. The value gap
favoring $a^R=\up$ over $a^R=\side$ is
\begin{equation*}
\kappa
:=
2(1-p)-2p
=
2-4p
=
\frac{2(1-3b(\Pgoal))}{1+b(\Pgoal)}
\in(0,2).
\end{equation*}
Therefore, $\lim_{\beta^H\to\infty}x=
(1+e^{\beta^R\kappa})^{-1}$. After observing $a^H=\up$, the value gap favoring $a^R=\side$ over $a^R=\up$ is 2, so $y=(1+e^{2\beta^R})^{-1}$.
These expressions yield $x=e^{-\beta^R\kappa}+o(e^{-\beta^R\kappa})$ and $y=e^{-2\beta^R}+o(e^{-2\beta^R})$.
Since $0<\kappa<2$, $e^{-2\beta^R} =
o(e^{-\beta^R\kappa})$,
so $x$ decays more slowly than $y$ and governs the leading-order scaling of both pedagogic advantages.

Thus, $A_{\Pgoal}
=
2e^{-\beta^R\kappa}
+
o(e^{-\beta^R\kappa})$ and $-A_{\Lgoal}
=
2e^{-\beta^R\kappa}
+
o(e^{-\beta^R\kappa})$.
Note that these quantities are not exactly equal; their difference, $A_{\Pgoal}-(-A_{\Lgoal})=4y$, is lower-order. Importantly, both advantages vanish at the same rate, so $a^H = \side$ unambiguously signals $\Pgoal$ when $\beta^H e^{-\beta^R\kappa}\to\infty$; a sufficient scaling
condition is
\begin{equation*}
\frac{1}{\beta^H}
=
o(e^{-\beta^R\kappa}),
\quad
b(\Pgoal)\in(0,\tfrac13);
\ \beta^R\to\infty.
\end{equation*}

\textbf{Benign regime ($b(\Pgoal)\in [\frac13, 1)$, $\beta^R\to\infty$).} No exponential growth of $\beta^H$ in $\beta^R$ is required.

\textbf{Hyperbolic regime ($\beta^R\to0$).}
Again let $p:=b^{R1+}_\infty(\Pgoal\mid\side)$ (\ref{eq:ioc_ceiling_tetris}), with fixed
$b(\Pgoal)\in(0,1)$. Expanding R1's softmax policy around $\beta^R=0$ gives
\[
x
=
\tfrac12
+
(p-\tfrac12)\beta^R
+
o(\beta^R),
\qquad
y
=
\tfrac12
-
\tfrac12\beta^R
+
o(\beta^R).
\]
Hence, $A_{\Pgoal}
=
2-2(1-p)\beta^R+o(\beta^R)$ and
$-A_{\Lgoal}
=
2p\beta^R+o(\beta^R)$. $a^H = \side$ therefore remains strongly preferred under $\Pgoal$ as $\beta^R \to 0$, since
$A_{\Pgoal}\to2$, but suppressing
$\pi^{H2}(\side\mid\Lgoal)$ requires $\beta^H\beta^R\to\infty$; a sufficient scaling condition is
\begin{equation*}
\frac{1}{\beta^H}=o(\beta^R),
\quad
\beta^R\to0.
\end{equation*}

\begin{figure}[!tb]
    \centering
    \includegraphics[width=0.8\linewidth]{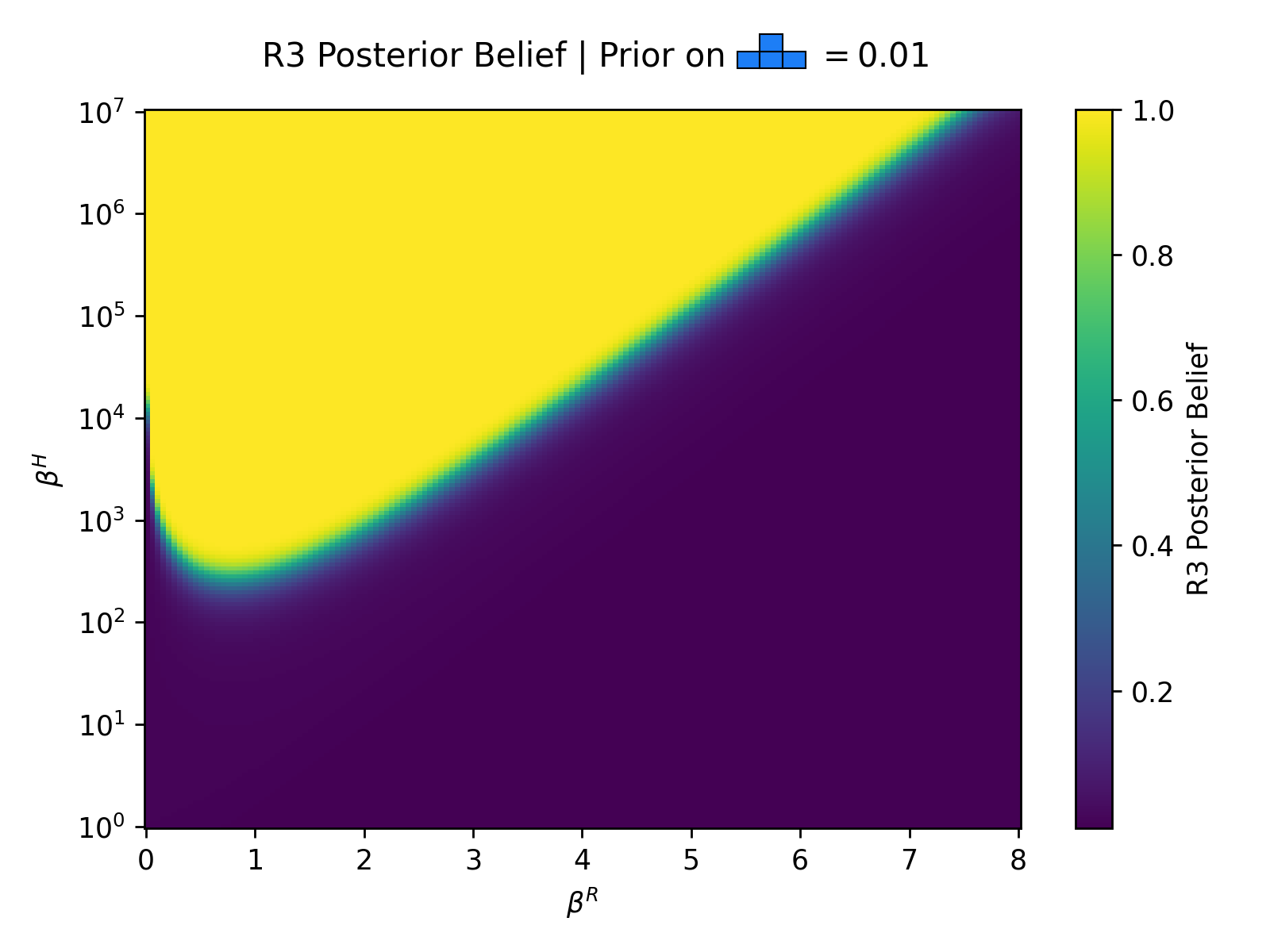}
    \caption{One-step pragmatic (R3) posterior
    $b^{R3+}(\protect\Pgoal\mid\protect\side)$ at
    $s=\protect\PartialUp$ as a function of $\beta^R, \beta^H$ with prior $b(\protect\Pgoal)=0.01$. The phase boundary from low to near-certain belief is consistent with the derived scaling regimes: hyperbolic growth as $\beta^R \to 0$ and exponential growth as $\beta^R \to \infty$.}
    \label{fig:belief_heatmap_R3}
\end{figure}

\subsection{Zero-Cost Pedagogy Enables Equilibrium-in-One}

We note that, in general, performing one round of pragmatic--pedagogic reasoning does not necessarily yield an optimal assistance game strategy pair. However, if the resulting R3--H2 policies attain the oracle value from the current state under every possible human goal, then they are necessarily optimal. Indeed, no team strategy can exceed $V^\mathrm{O}(s;\theta)$, which upper-bounds the optimal value of the assistance game, $V^{\mathcal G}(s, b;\theta)$.

Building on these insights, we identify a class of games in which the human and the robot arrive at an unambiguous pedagogic leverage map at \textit{zero cost} to the human: signaling any goal requires no sacrifice in value. In these games, pedagogy is ``free,'' allowing the human--robot team to achieve exactly the oracle value from state $s$ as if the robot already knew the human's true goal. We formalize this property below.

\begin{definition}[Zero-cost pedagogy]
\label{def:free-pedagogy}
A fully populated leverage map $\mathcal{L}$ achieves zero-cost pedagogy at state $s$
if, for each possible human goal $\theta\in\Theta$,
\begin{equation*}
    Q^\mathrm{O}(s, a^H, a^{R,\mathrm{O}};\theta) = V^\mathrm{O}(s; \theta) \quad \forall \, a^H \in \mathcal{L}(\theta),
\end{equation*}
where $a^{R,\mathrm{O}} \in \arg\max_{a^R\in\mathcal A^R(s)} Q^\mathrm{O}(s,a^H,a^R; \theta)$.
\end{definition}



The relevant class of assistance games immediately follows.

\begin{definition}[Action-separable assistance game] \label{def:action_separable}
    Let $\mathcal L^{H2}_{s,b}$ denote the pedagogic leverage map as $\beta^H,\beta^R\to\infty$. An assistance game $\mathcal G$ is action-separable at $(s,b)$ if $\mathcal L^{H2}_{s,b}$ is fully populated and achieves zero-cost pedagogy (\Cref{def:free-pedagogy}).
\end{definition}

We now prove that, for an action-separable assistance game, one round of pragmatic--pedagogic reasoning suffices to reach an optimal equilibrium solution. As we show in the next section, this solution is easily computable (\Cref{alg:ppbr_one_round}).

\begin{theorem}[Pragmatic--pedagogic equilibrium in one best response] \label{thm:equilibrium}
Fix a state $s$ and prior $b$ over $\Theta$. Suppose that the assistance game is action-separable at $(s, b)$ (\Cref{def:action_separable}). Then, the limiting rational ($\beta^H, \beta^R \to \infty$) R3--H2 policies ($\pi^{R3\star}, \pi^{H2\star}$) constitute an optimal pragmatic--pedagogic equilibrium of the assistance game at $(s, b)$.
\end{theorem}
\begin{proof}
Fix any true goal $\theta^\dagger\in\Theta$. Since the game is action-separable, the pedagogic leverage map is fully populated, so $\mathcal L^{H2}_{s,b}(\theta^\dagger)$ is nonempty; fix any human action $a^\dagger\in\mathcal L^{H2}_{s,b}(\theta^\dagger)$. By \Cref{prop:leverage_full_alignment}, observing $a^\dagger$ concentrates R3's posterior on $\theta^\dagger$ as $\beta^H \to \infty$. Because $\mathcal L^{H2}_{s,b}$ achieves zero-cost pedagogy (\Cref{def:free-pedagogy}), $a^\dagger$ and R3's limiting rational response jointly attain the oracle value from state $s$:
\begin{equation*}
\expect_{a^R \sim \pi^{R3\star}(\cdot \mid s,b^{R3+}_\infty,a^{\dagger})}\left[Q^\mathrm{O}(s,a^\dagger,a^R;\theta^\dagger)\right] = \max_{a^R\in\mathcal A^R(s)}\left[Q^\mathrm{O}(s,a^\dagger,a^R;\theta^\dagger)\right] = V^\mathrm{O}(s;\theta^\dagger).
\end{equation*}
Since H2's limiting rational policy is supported exclusively on $\mathcal L^{H2}_{s,b}(\theta^\dagger)$, averaging over human actions $a^\dagger \sim \pi^{H2\star}(\cdot \mid s,b;\theta^\dagger)$ also yields $V^\mathrm{O}(s;\theta^\dagger)$.

We emphasize that, upon observing any leveraging human action, R3's uncertainty \textit{really does} collapse in a single time step, after which the strategies for the rest of the time horizon correspond to the oracle team policies and realize $V^\mathrm{O}(s;\theta^\dagger)$ in full. Since $\theta^\dagger$ was arbitrary, the limiting R3--H2 policies attain the oracle value under every possible goal. No feasible strategy pair can exceed the fully informed oracle value under any goal, so $(\pi^{R3\star},\pi^{H2\star})$ is optimal.

Furthermore, assistance games are common-payoff games, so any unilateral policy deviation yields another feasible strategy pair evaluated under the same team objective. Holding $\pi^{R3\star}$ fixed, no human deviation can yield value greater than the globally optimal value already attained by
$(\pi^{R3\star},\pi^{H2\star})$. Hence,
$\pi^{H2\star}$ is a best response to $\pi^{R3\star}$. The same
argument applies to any unilateral robot deviation, so
$\pi^{R3\star}$ is a best response to $\pi^{H2\star}$.

Therefore, the limiting rational ($\beta^H, \beta^R \to \infty$) R3--H2 policies constitute an optimal pragmatic--pedagogic equilibrium of the assistance game at $(s,b)$.
\qed
\end{proof}

\paragraph{Running example:} From state $s=\PartialUp$ and any prior $b(\Pgoal)\in(0,1)$, recall that the pedagogic leverage map $\mathcal L^{H2}_{s,b}$ is fully populated as $\beta^H, \beta^R \to \infty$: $\mathcal L^{H2}_{s,b}(\Pgoal)=\{\side\}$ and $\mathcal L^{H2}_{s,b}(\Lgoal)=\{\up\}$. By \Cref{prop:leverage_full_alignment}, $b^{R3+}_\infty(\Pgoal \mid \side) = 1$ and $b^{R3+}_\infty(\Lgoal \mid \up) = 1$. Evaluating R3's limiting rational policy gives
\begin{align*}
\E_{a^R\sim\pi^{R3\star}(\cdot\mid s,b^{R3+}_\infty,\side)}
\!\left[Q^\mathrm{O}(s,\side,a^R;\Pgoal)\right]
&= 2 = V^\mathrm{O}(s; \Pgoal), \\
\E_{a^R\sim\pi^{R3\star}(\cdot\mid s,b^{R3+}_\infty,\up)}
\!\left[Q^\mathrm{O}(s,\up,a^R;\Lgoal)\right]
&= 2 = V^\mathrm{O}(s; \Lgoal).
\end{align*}
These calculations confirm that $\mathcal L^{H2}_{s,b}$ achieves zero-cost pedagogy (\Cref{def:free-pedagogy}). Consequently, the Tetris assistance game is action-separable at $(s,b)$ (\Cref{def:action_separable}), so the limiting rational ($\beta^H, \beta^R \to \infty$) R3--H2 policies constitute an optimal pragmatic--pedagogic equilibrium (\Cref{thm:equilibrium}).

\section{Algorithm}

Pragmatic--Pedagogic Best Response (PPBR) provides a simple procedure for computing an optimal equilibrium solution to an action-separable assistance game (\Cref{def:action_separable}). \Cref{alg:ppbr_one_round} checks whether the assistance game is action-separable at $(s,b)$ and, if so, returns the pedagogic leverage map $\mathcal{L}$ (\Cref{def:ped_leverage}), assigning to each goal the corresponding set of leveraging human actions.

Since the sets $\{\mathcal{L}(\theta)\}_{\theta\in\Theta}$ are disjoint, each leveraging action identifies a unique goal, inducing an inverse leverage map $\mathcal{L}^{-1}: \, \mathcal{A}^\mathcal{L} \to \Theta$, where $\mathcal{A}^\mathcal{L}$ is the union of all leverage sets $\mathcal{L}(\theta)$.
The limiting H2 policy, R3 posterior belief, and R3 policy are then readily obtained as
\begin{align*}
\pi^{H2\star}(a^H \mid s, b; \theta)
&= \frac{\mathbbm{1}\!\big[a^H \in \mathcal{L}(\theta)\big]}{|\mathcal{L}(\theta)|}, \\
b^{R3+}_\infty(\theta \mid a^H)
&= \mathbbm{1}\!\big[\theta = \mathcal{L}^{-1}(a^H)\big], \\
\pi^{R3\star}(a^R \mid s, b^{R3+}_\infty, a^H)
&= \pi^{R, \mathrm{O}}\!\big(a^R \mid s, a^H;\, \mathcal{L}^{-1}(a^H)\big),
\end{align*}
where $\pi^{R, \mathrm{O}}$ is the oracle robot policy, that is,
\begin{equation*}
\pi^{R, \mathrm{O}}(a^R \mid s, a^H; \theta)
\propto \mathbbm{1}\!\left[a^R \in \arg\max_{a \in \mathcal{A}^R(s)} Q^\mathrm{O}(s, a^H, a; \theta)\right].
\end{equation*}

The runtime of \Cref{alg:ppbr_one_round} is linear in the size of the goal set and action sets: $O(|\Theta||\mathcal{A}^H||\mathcal{A}^R|)$, which is as efficient as solving the oracle MDP. While our procedure directly computes the limiting behavior as $\beta^H, \beta^R \to \infty$ in our Tetris running example, plugging in finite values of $\beta^H$ (with $\beta^R = 1$) generates the curves depicted in \Cref{fig:prior-to-posterior}.

\begin{algorithm}[t]
\caption{Pragmatic--Pedagogic Best Response}
\label{alg:ppbr_one_round}
\KwIn{state $s$, prior $b$, finite $\beta^R > 0$ (hyperparameter)}
\KwOut{action-separable (Boolean), pedagogic leverage map $\mathcal{L}$ (\Cref{def:ped_leverage})}
Compute the limiting rational solipsistic human policy $\pi^{H0\star}$ at $(s,b)$ as in~\eqref{eq:pi_h0}\;
\ForEach{candidate human action $a^H$}{
    Compute the limiting literal posterior $b_\infty^{R1+}(\cdot \mid a^H)$\;
    Compute the noisy robot response policy $\pi^{R1}$ as in~\eqref{eq:pi_r1}\;
}
Compute the induced pedagogic human values $Q^{H2}$ from $\pi^{R1}$\;
\ForEach{goal $\theta$ 
    $\in$
    $\Theta$}{
    $\mathcal{L}(\theta) \gets \arg\max_{a^H \in \mathcal{A}^H(s)} Q^{H2}(s, b, a^H; \theta)$\;
    \If{$\exists \, a^H \in \mathcal{L}(\theta)$ already assigned to another goal}{
        \Return (False, $\emptyset$)\;
    }
    \If{$\exists\,a^H\in\mathcal L(\theta)$ such that
    $\displaystyle\max_{a^R\in\mathcal A^R(s)}
    Q^\mathrm{O}(s,a^H,a^R;\theta)
    < V^\mathrm{O}(s;\theta)$}{
        \Return (False, $\emptyset$)\;
    }
}
\Return (True, $\mathcal{L}$)\;
\end{algorithm}

\begin{figure}[!tb]
  \centering
  \includegraphics[width=0.62\textwidth]{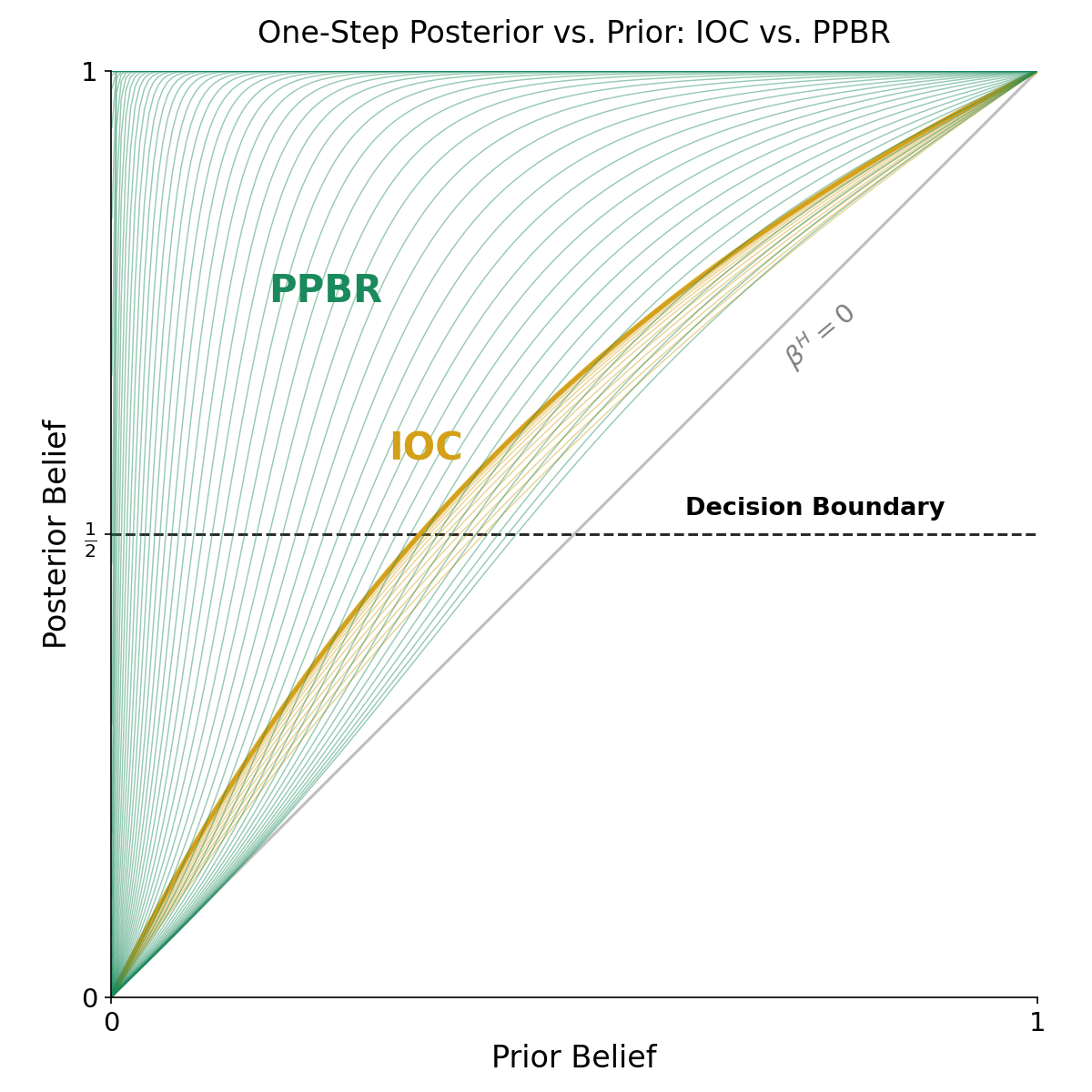}
  \caption{One-step posterior $b^+(\protect\Pgoal \mid  \protect\side)$ vs.\ prior $b(\protect\Pgoal)$ at $s = \protect\PartialUp$ with $\beta^R=1$. IOC (R1, orange) saturates at the ceiling $b^{R1+}_\infty = 2b/(1+b)$ as $\beta^H$ increases. PPBR (R3, green) instead breaks the ceiling and approaches $b^{R3+}_\infty = 1$.}
  \label{fig:prior-to-posterior}
\end{figure}

\section{Limitations and Future Work}

A direct extension is \emph{split leverage}, in which a single action is H2-optimal for more than one goal, giving only \textit{partial} disambiguation. Analyzing the resulting class of assistance games and lifting \Cref{alg:ppbr_one_round} to a multi-step procedure is beyond the scope of this paper. More broadly, we have shown that nested reasoning with Boltzmann-rational models naturally breaks the ambiguity between otherwise task-equivalent actions. We note that explicit level-$k$ iteration may not be the only algorithmic route to the pragmatic--pedagogic solution; investigating other approaches is left to future work.

\section{Conclusion}

In this paper, we identified a class of assistance games in which the human and the robot arrive at an unambiguous correspondence between goals and actions at \textit{zero cost} to the human, who can convey and achieve any goal without sacrificing task value. For these action-separable games, pragmatic--pedagogic reasoning resolves goal uncertainty in a single time step, breaking the \gls{ioc} inference ceiling that leads to incorrigible robot behavior. We proved that one best-response round starting from a mainstream \gls{ioc} solution suffices to reach an optimal equilibrium of the full-horizon game. Pragmatic--Pedagogic Best Response computes this equilibrium as efficiently as solving the oracle MDP, yielding a corrigible robot assistant that maximally empowers the human user. Ultimately, robust alignment hinges on how robotic and AI systems represent and interpret the humans with whom they interact.

\begin{credits}
\subsubsection{\ackname} We extend a special thank you to Donggeon Oh and Tom Silver for insightful discussions and feedback.

\subsubsection{\discintname}
The authors have no competing interests to declare that are relevant to the content of this article.
\end{credits}
%
%

\printbibliography
\end{document}